%% file: main.tex
\documentclass{article}

\usepackage[preprint]{neurips_2026}

\usepackage[utf8]{inputenc}
\usepackage[T1]{fontenc}
\usepackage{hyperref}
\usepackage{url}
\usepackage{booktabs}
\usepackage{amsfonts}
\usepackage{amsmath}
\usepackage{amssymb}
\usepackage{amsthm}
\usepackage{nicefrac}
\usepackage{microtype}
\usepackage{xcolor}
\usepackage{graphicx}
\usepackage{tikz}
\usepackage{pgfplots}
\pgfplotsset{compat=1.18}
\usetikzlibrary{positioning, calc, fit, shapes.geometric, arrows.meta, backgrounds, decorations.pathreplacing, patterns}
\usepackage{algorithm}
\usepackage{algorithmic}
\usepackage{wrapfig}
\usepackage{multirow}
\usepackage{makecell}
\usepackage{caption}
\usepackage{subcaption}
\usepackage{tcolorbox}

\newcommand{\pref}{\mathbf{p}}
\newcommand{\rawpref}{\tilde{\mathbf{p}}}
\newcommand{\predpref}{\hat{\mathbf{p}}}
\newcommand{\delpref}{\boldsymbol{\Delta}}
\newcommand{\enc}{\mathrm{enc}}
\newcommand{\predfn}{f_{\theta}}
\newcommand{\method}{\textsc{Latte}}
\newcommand{\static}{\boldsymbol{\pi}}
\newcommand{\state}{\mathbf{s}}
\newcommand{\content}{\mathbf{c}}

\newtheorem{proposition}{Proposition}

\title{\method: Forecasting Peer Anchored Preference Trajectories for Personalized LLM Generation}

\author{%
  Jinze Li$^{1}$, Xiaoyan Yang$^{2}$, Shuo Yang$^{1}$, Jinfeng Xu$^{1}$,\\
  \textbf{Yue Shen$^{2}$, Jian Wang$^{2}$, Jinjie Gu$^{2}$, Edith Cheuk-Han Ngai$^{1,\dagger}$}\\
  \normalfont $^{1}$The University of Hong Kong \quad
  $^{2}$Ant Healthcare, Ant Group\\
  \normalfont $^\dagger$Corresponding authors.
}
\hypersetup{hidelinks}
\begin{document}

\maketitle

\begin{abstract}
Personalized generation with frozen large language models requires a conditioning signal that is both compact and current. Existing personalization methods typically retrieve or summarize user histories in text, or compress them into static latent profiles and soft prompts. These approaches are efficient, but they treat a user's past behavior as an aggregate profile and therefore mix stable identity, recent drift, and item content in the same representation. We propose \textbf{LA}tent \textbf{T}rajectory \textbf{T}racking and \textbf{E}xtrapolation (\method{}), a framework that represents personalization as forecasting a peer anchored relative preference state. For each historical session, \method{} subtracts a time masked baseline formed from comparable users who responded to the same item, producing a state that measures how the target user differs from peers under a shared item context. A lightweight sequence predictor then forecasts the next state in this trajectory, and a State to Token Bridge injects the forecast into a frozen instruction tuned LLM through a single anchored soft token. We provide a latent factor analysis showing when peer anchoring cancels shared item variation and why temporal forecasting trades off stale averages against noisy recent states. Experiments on Amazon Reviews 2023 and MemoryCD show that \method{} consistently outperforms retrieval, summary memory, static latent profiles, difference aware latent profiles, and soft prompt compression baselines. On Amazon Reviews 2023, \method{} improves average ROUGE-L from 0.219 for a static latent profile and 0.245 for the strongest added latent compression baseline to 0.259. Additional pairwise comparisons and diagnostic analyses suggest that the improvement is mainly due to forecasting user-specific trajectory information, rather than merely adding a soft prompt interface.
\end{abstract}

\section{Introduction}
\label{sec:intro}

Large language models (LLMs)~\citep{dubey2024llama3} are increasingly used in settings where the same input should lead to different outputs for different users. A review assistant should reflect a user's writing style. A recommendation explanation should emphasize the criteria that the user cares about. A long running conversational agent should adapt as a user changes goals, tone, or interests across sessions. These scenarios raise a basic question for personalized generation. What should a frozen LLM condition on when it generates for a particular user at a particular moment?

Most existing systems answer this question with a user profile. Prompt based methods retrieve or summarize previous interactions and place the resulting text in the context~\citep{salemi2024lamp,kumar2024longlamp,mysore2023pearl,salemi2024optimization}. Latent methods compress the history into an embedding, a steering direction, a user module, or a small set of soft prompt vectors~\citep{qiu2025dep,qiu2025dpl,hebert2024persoma,liu2024pplug,ning2024userllm}. These approaches are compact and often effective, but they usually treat the user as a static object. They aggregate past behavior into one representation and reuse it for future generation.

Static aggregation is limiting because the relevant user signal is often temporal. A reviewer may become more technical after years of short impressions. A reader may move from genre fiction to literary criticism. A conversational user may revise constraints over several sessions. In such cases, the useful signal is not only what the user has preferred on average, but where the user appears to be now. This distinction is especially important for frozen LLM personalization, because the model can often generate fluent and item relevant text from the target metadata alone. A stale or content dominated profile may therefore look plausible while failing to match the user's current behavior.

This paper studies personalized generation as latent state forecasting. Instead of compressing the full history into a single profile, we construct a sequence of relative preference states and forecast the state that should condition the next generation. This view separates three problems that static profiles merge. First, each historical response should be converted into a state that is comparable across different items. Second, the current state should be predicted from the ordered trajectory rather than estimated by an unordered average. Third, the predicted state should be injected into a frozen LLM without adding user specific parameters.

We propose \textbf{LA}tent \textbf{T}rajectory \textbf{T}racking and \textbf{E}xtrapolation (\method{}). For each historical session, \method{} forms a time masked peer baseline from comparable users who responded to the same item before the target timestamp. It subtracts this baseline from the target user's response embedding and normalizes the residual. The resulting vector is a peer anchored relative state. It asks how the user responded relative to similar peers under the same item context, which reduces shared item variation before temporal modeling.

Given the sequence of peer anchored states, \method{} trains a lightweight predictor to forecast the next state with a direct regression objective. This decouples state prediction from the language modeling loss. The separation is useful because generation loss alone can allow the conditioning vector to collapse into a low rank shortcut while the frozen LLM relies on item metadata. After forecasting, a State to Token Bridge maps the predicted state into the token embedding space of the frozen LLM. At inference time, the bridge replaces one placeholder token, and a natural language anchor tells the model how to interpret the injected state.

We evaluate \method{} on Amazon Reviews 2023 and MemoryCD against retrieval, summary memory, static latent profiles, recent and time decayed latent profiles, a DEP style difference aware static profile, and a PERSOMA style soft prompt compression baseline. The strongest comparisons are therefore not only against simple static profiles, but also against latent compression and difference aware user modeling. Across datasets, \method{} improves lexical overlap and history aware preference judgments. On Amazon Reviews 2023, it improves average ROUGE-L from .245 for the strongest added latent compression baseline to .259. Preference fidelity metrics, peer leakage controls, bootstrap intervals, and collapse diagnostics indicate that the gains come from forecasting user specific trajectory information rather than from the soft prompt slot alone.

Our contributions are as follows.
\begin{itemize}
    \item We formulate frozen LLM personalization as forecasting a peer anchored relative preference state, turning user history from a static profile into a time ordered latent trajectory.
    \item We introduce \method{}, a modular framework that constructs same item peer residual states, predicts the next state with a lightweight sequence model, and injects the forecast into a frozen LLM through one anchored soft token.
    \item We provide analytical and empirical evidence that peer anchoring reduces shared item variation, trajectory forecasting improves over static latent compression, and the observed gains are not explained by peer leakage, bridge mismatch, or representation collapse.
\end{itemize}

\section{Related Work}
\label{sec:related}

\textbf{Personalized generation and retrieval profiles.} Prompt based personalization conditions LLMs on user histories, retrieved examples, or textual summaries. LaMP~\citep{salemi2024lamp}, LongLaMP~\citep{kumar2024longlamp}, PEARL~\citep{mysore2023pearl}, and retrieval optimization methods~\citep{salemi2024optimization} establish retrieval as a strong baseline. Recent benchmarks make the evaluation more demanding. PersonalLLM studies individual preference variation at scale~\citep{zollo2025personallm}. PrefEval evaluates whether LLMs infer and follow user preferences in long multi session conversations~\citep{zhao2025prefeval}. HYDRA factorizes black box personalization into shared and user specific components over retrieved histories~\citep{zhuang2024hydra}. These works motivate our use of retrieval and history aware evaluation. \method{} differs by replacing retrieved or summarized text with a forecast latent state.

\textbf{Latent personalization.} Latent methods compress user information into embeddings, soft prompts, steering vectors, or user modules. PERSOMA compresses extensive history into soft prompt embeddings~\citep{hebert2024persoma}. PPlug and User-LLM introduce plug in user profile representations~\citep{liu2024pplug,ning2024userllm}. DEP and DPL show that inter user differences are useful for personalization~\citep{qiu2025dep,qiu2025dpl}. Personalized steering vectors and parameter efficient adaptation provide alternative latent control mechanisms~\citep{cao2024bipo,tan2024perlora}. This line of work is closest to ours. We use the same core insight that relative signals can be more informative than user only signals. The key difference is temporal. DEP and DPL construct static difference aware representations from selected or aggregated histories. \method{} constructs a sequence of peer anchored states and forecasts the next state before generation. Our DEP style baseline uses the same peer anchored states but averages them, which isolates the contribution of trajectory forecasting.

\textbf{Long horizon memory.} Long term memory benchmarks such as LoCoMo~\citep{maharana2024locomo}, LongMemEval~\citep{wu2025longmemeval}, PersonaMem~\citep{lin2025personamem}, PerLTQA~\citep{du2024perltqa}, PrefEval~\citep{zhao2025prefeval}, and MemoryCD~\citep{zhang2026memorycd} document failures of long context and retrieval based personalization. Memory architectures store hidden states, key value memories, or retrieved chunks, as in Memorizing Transformers~\citep{wu2022memorizing}, Recurrent Memory Transformer~\citep{bulatov2022rmt}, RETRO~\citep{borgeaud2022retro}, and MEMORYLLM~\citep{wang2024memoryllm}. \method{} addresses a complementary bottleneck. Instead of storing more text or activations, it learns a compact current preference state that can be injected through one token.

\textbf{Sequential user modeling.} Sequential recommendation models evolving user behavior with recurrent, attentive, bidirectional, diffusion, and instance adaptive architectures~\citep{hidasi2016gru4rec,kang2018sasrec,sun2019bert4rec,yang2023dreamrec,kong2024ilora}. Time series forecasting studies long horizon prediction with decomposed Transformers, frequency models, patching, and task general temporal backbones~\citep{wu2021autoformer,zhou2022fedformer,nie2023patchtst,wu2023timesnet,zeng2023dlinear}. We borrow the trajectory view, but the predicted object is not the next item and not a scalar time series. It is a peer normalized latent state used to condition a frozen language model.

\section{Method}
\label{sec:method}

\begin{figure}[t]
  \centering
  \includegraphics[width=0.8\linewidth]{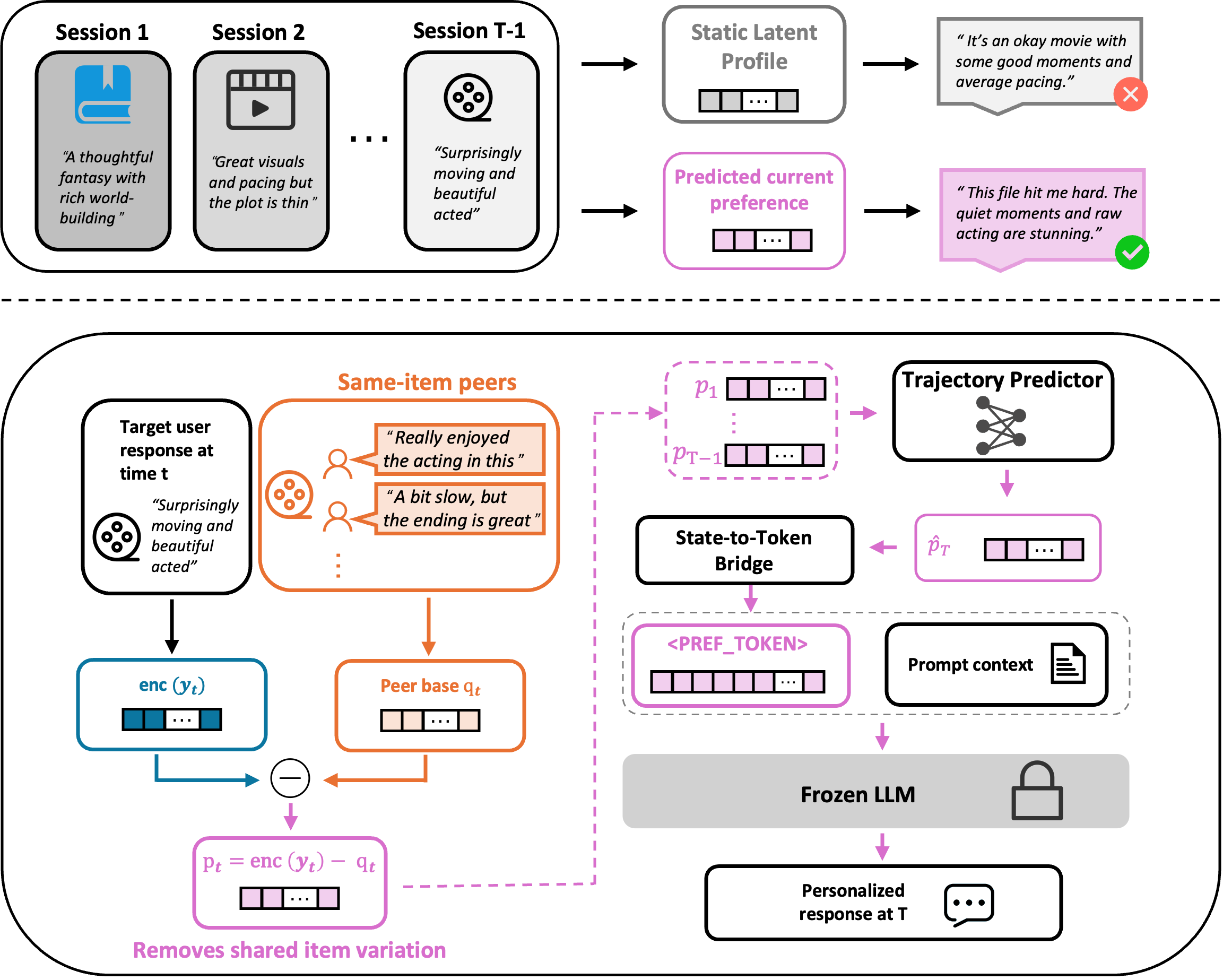}
  \caption{\textbf{LATTE forecasts peer anchored preference trajectories for personalized generation.}
Top: a static latent profile aggregates the user’s history into one vector and can miss recent preference shifts, while LATTE forecasts the user’s current preference state. Bottom: LATTE first constructs peer anchored relative states from historical sessions, then uses a trajectory predictor to forecast the current state, and finally injects the forecast into a frozen LLM through a State-to-Token Bridge.}
  \label{fig:overview}
  \vspace{-1em}
\end{figure}

\method{} has three stages. Stage 1 constructs one peer anchored relative state for each historical session. Stage 2 trains a predictor to forecast the next state from the observed trajectory. Stage 3 maps the forecast into the hidden dimension of a frozen LLM with a State-to-Token Bridge and injects it as one anchored soft prompt token. The stages are trained separately, which makes the representation target, predictor, and bridge module independently testable.

\subsection{Preliminaries}
\label{sec:method-pre}

\textbf{Dynamic personalization setting.} A user $u$ has a chronological history $\mathcal{H}_{u,T-1}=\{(i_1,u_1,\tau_1),\ldots,(i_{T-1},u_{T-1},\tau_{T-1})\}$, where $i_t$ is an item or context, $u_t$ is the user's textual response, and $\tau_t$ is the timestamp. At time $T$, the model receives target metadata $x_T$ for item $i_T$ and must generate $y_T$ in the user's current style and preference state. We train and evaluate with temporal splits, where later responses are never used to build earlier histories.

\textbf{Encoder and base model.} Let $\enc(\cdot)$ be a frozen sentence encoder that maps text to a $d$ dimensional unit norm embedding. We use bge m3~\citep{chen2024bgem3} with $d=1024$. The base generator is a frozen instruction tuned LLM $\mathcal{M}$ with token embedding matrix $E\in\mathbb{R}^{|V|\times h}$. Personalization is performed by adding a placeholder token \texttt{[PREF\_TOKEN]} and replacing its embedding at runtime. All LLM weights remain frozen.

\textbf{Static latent profile.} The standard latent profile compresses the history into a single vector
\begin{equation}
\static(u)=A\big(\enc(u_1),\ldots,\enc(u_{T-1})\big),
\label{eq:static}
\end{equation}
where $A$ may be mean pooling, attention pooling, or a learned encoder~\citep{qiu2025dep,hebert2024persoma}. \method{} replaces this static object with a sequence $\pref_1(u),\ldots,\pref_{T-1}(u)$ and a forecast $\predpref_T(u)$.

\textbf{Notation.} We write $\rawpref_t(u)$ for the unnormalized peer anchored residual, $\pref_t(u)\in\mathbb{R}^{d}$ for its normalized state, $\bar{q}_t(u)$ for the peer baseline, $\delpref_t(u)=\pref_t(u)-\pref_{t-1}(u)$ for an adjacent change, and $\predpref_T(u)$ for the predicted current state. Bold lowercase symbols denote vectors.

\subsection{Latent Preference Trajectory}
\label{sec:method-pref}

\textbf{Peer anchored state.} Raw response embeddings are not suitable trajectory states because their changes often reflect item changes. We therefore construct a relative state for each historical session. Let $\mathcal{N}_m(i_t,t)$ be up to $m$ peer users who reviewed the same item $i_t$ before timestamp $\tau_t$. The target user is excluded. Peers with fewer than four earlier interactions are excluded. Let $r_{v,i_t}$ be peer $v$'s response to item $i_t$. We define the earlier interaction set $\mathcal{I}_{u,t}=\{\ell:\tau_\ell<\tau_t\}$ and use the profile summary
\begin{equation}
\phi_u(t)=
\begin{cases}
\frac{\sum_{\ell\in\mathcal{I}_{u,t}}\enc(u_\ell)}{\left\|\sum_{\ell\in\mathcal{I}_{u,t}}\enc(u_\ell)\right\|_2}, & \mathcal{I}_{u,t}\neq\emptyset,\\[6pt]
\mathbf{0}, & \mathcal{I}_{u,t}=\emptyset.
\end{cases}
\label{eq:profile-summary}
\end{equation}
If either the target profile or all peer profiles are zero, we use uniform weights. Otherwise, peer weights are computed with temperature $\gamma$,
\begin{equation}
 w_{u,v,t}=\frac{\exp(\gamma\langle \phi_u(t),\phi_v(t)\rangle)}{\sum_{v'\in\mathcal{N}_m(i_t,t)}\exp(\gamma\langle \phi_u(t),\phi_{v'}(t)\rangle)}.
\label{eq:peer-weights}
\end{equation}
The baseline, residual, and normalized state are
\begin{equation}
\bar{q}_t(u)=\sum_{v\in\mathcal{N}_m(i_t,t)}w_{u,v,t}\,\enc(r_{v,i_t}),\quad
\rawpref_t(u)=\enc(u_t)-\bar{q}_t(u),\quad
\pref_t(u)=\frac{\rawpref_t(u)}{\|\rawpref_t(u)\|_2}.
\label{eq:pref}
\end{equation}
The peer baseline is used only for historical sessions. At test time, no peer review of the target item $i_T$ is placed in the LLM prompt or used to construct $\predpref_T(u)$.

\textbf{Why subtraction helps.} The following proposition states the sense in which peer anchoring removes shared item variation. It applies to the raw residual $\rawpref_t$. The predictor uses the normalized direction $\pref_t$, which preserves the relative direction of $\rawpref_t$ when the residual norm is nonzero.

\begin{proposition}[Peer anchoring under an additive embedding model]
\label{prop:anchor}
Condition on the peer set and weights in Eq.~\ref{eq:peer-weights}. Assume a response embedding for user $u$ on item $i_t$ has the form
\begin{equation}
\enc(u_t)=\content_{i_t,t}+\state_{u,t}+\boldsymbol{\epsilon}_{u,t},
\label{eq:additive-model}
\end{equation}
where $\content_{i_t,t}$ is an item component shared by users at time $t$, $\state_{u,t}$ is a user state, and $\boldsymbol{\epsilon}_{u,t}$ is zero mean noise independent across users after conditioning on the weights. Assume each peer response has the same item component and peer state $\state_{v,t}$. Then
\begin{equation}
\mathbb{E}\big[\rawpref_t(u)\mid \state_{u,t},\{\state_{v,t}\}_{v\in\mathcal{N}_m}\big]
=\state_{u,t}-\sum_{v\in\mathcal{N}_m(i_t,t)}w_{u,v,t}\state_{v,t}.
\label{eq:anchor-expectation}
\end{equation}
Thus the shared item term is removed in expectation. If item variation has covariance $\Sigma_c$ and encoder noise has covariance $\sigma^2 I$, the raw embedding includes the additional covariance $\Sigma_c$, while the anchored residual has noise covariance $\sigma^2(1+\sum_v w_{u,v,t}^2)I$ around the relative state.
\end{proposition}

\begin{proof}
Substitute Eq.~\ref{eq:additive-model} into Eq.~\ref{eq:pref}. The weighted peer baseline contains $\content_{i_t,t}$ because the weights sum to one. The item component cancels, and the remaining expected value is the user state minus the weighted peer state. The covariance comparison follows from the independence of the item component and the zero mean encoder noise.
\end{proof}

\textbf{Why forecasting helps.} A static profile can be optimal when the state is stationary and noise is high, but it becomes temporally stale under drift. The next proposition separates this bias effect from the variance reduction obtained by averaging.

\begin{proposition}[Bias and variance under local linear drift]
\label{prop:drift}
Assume normalized relative states are locally approximated by $\pref_t=\mathbf{a}+t\mathbf{g}+\boldsymbol{\eta}_t$, where $\mathbb{E}[\boldsymbol{\eta}_t]=0$ and $\boldsymbol{\eta}_t$ is independent across sessions with covariance $\sigma^2 I$. Let $\boldsymbol{\mu}_T=\mathbf{a}+T\mathbf{g}$ be the conditional mean next state. The static average estimator $\bar{\pref}=\frac{1}{T-1}\sum_{t=1}^{T-1}\pref_t$ has squared bias $\frac{T^2}{4}\|\mathbf{g}\|_2^2$ and variance trace $d\sigma^2/(T-1)$ for estimating $\boldsymbol{\mu}_T$. The last state estimator $\pref_{T-1}$ has squared bias $\|\mathbf{g}\|_2^2$ and variance trace $d\sigma^2$. Thus the static average can have lower mean squared error than the last state when drift is small relative to observation noise, while an order aware linear forecast can remove the static lag under the local linear model.
\end{proposition}

\begin{proof}
The expectation of the static average is $\mathbf{a}+\frac{T}{2}\mathbf{g}$, while the target mean is $\mathbf{a}+T\mathbf{g}$. Its squared bias is therefore $\|\frac{T}{2}\mathbf{g}\|_2^2$, and its variance trace is $d\sigma^2/(T-1)$. The expectation of the last state is $\mathbf{a}+(T-1)\mathbf{g}$, so its squared bias is $\|\mathbf{g}\|_2^2$, and its variance trace is $d\sigma^2$. For $T-1\geq 2$, an ordinary least squares extrapolator fitted to the ordered pairs $(t,\pref_t)$ is unbiased for $\boldsymbol{\mu}_T$ under this model because the design contains both an intercept and the time index.
\end{proof}

\noindent\textbf{Remark.} Proposition~\ref{prop:drift} is a bias and variance statement, not a claim that the most recent state should always beat an average. Averaging can improve downstream generation when residual states are noisy or when stable user traits dominate short term drift. This is why the experiments compare static averages, last state forecasts, exponential smoothing, learned attention, and recurrent predictors. The learned predictors are intended to use order while still smoothing across multiple observations.

\textbf{Temporal prediction target.} Adjacent changes decompose as
\begin{equation}
\delpref_t(u)=\pref_t(u)-\pref_{t-1}(u).
\label{eq:delta}
\end{equation}
The subtraction in Eq.~\ref{eq:pref} reduces the item component before the predictor models these changes. The current state is forecast as
\begin{equation}
\predpref_T(u)=\predfn\big(\pref_1(u),\ldots,\pref_{T-1}(u)\big).
\label{eq:predict}
\end{equation}

\subsection{Trajectory Forecasting}
\label{sec:method-predictor}

\textbf{Objective.} The predictor is trained offline to regress to a held out constructed state $\pref_T(u)$ from the next chronological session. This target is a well defined derived statistic, namely the next response embedding after peer normalization. It is not a directly observed latent quantity. The loss is
\begin{equation}
\mathcal{L}_{\mathrm{pred}}(\theta)=\mathbb{E}_{u,T}\left[1-\cos\big(\predpref_T(u),\pref_T(u)\big)+\lambda\left\|\predpref_T(u)-\pref_T(u)\right\|_2^2\right].
\label{eq:loss-pred}
\end{equation}
Rolling temporal holdouts provide training pairs. A prefix predicts the next session state, and validation and test responses are never used to construct earlier histories, peer weights, or predictor inputs. Predictor training starts once at least four constructed states are available. Earlier sessions still provide states for later prefixes.

\textbf{Predictor family.} We evaluate six predictors. P0 uses the last state, $\predpref_T=\pref_{T-1}$. P1 uses a linear trend, $\predpref_T=\pref_{T-1}+(\pref_{T-1}-\pref_{T-2})$. P2 is an exponential moving average, $\predpref_T=\beta\pref_{T-1}+(1-\beta)\bar{\pref}$, where $\bar{\pref}=\frac{1}{T-1}\sum_{t=1}^{T-1}\pref_t$. P3 is learned attention pooling with the most recent state as query. It scores each previous state by
\begin{equation}
\begin{aligned}
 s_t &= \mathbf{v}_a^\top\tanh(W_h\pref_t+W_q\pref_{T-1}+\mathbf{b}_a),\\
 \alpha_t &= \frac{\exp(s_t)}{\sum_{\ell=1}^{T-1}\exp(s_\ell)},\\
 \predpref_T &= \mathrm{norm}\left(W_o\,\mathrm{concat}\left(\sum_{t=1}^{T-1}\alpha_t\pref_t,\pref_{T-1}\right)\right).
\end{aligned}
\label{eq:attention-predictor}
\end{equation}
where $\mathrm{norm}$ denotes unit normalization. P4 is a one layer GRU~\citep{cho2014gru} followed by a linear head. P5 is a two layer Transformer encoder~\citep{vaswani2017attention}. All predictors output one $d$ dimensional vector.

\subsection{State-to-Token Bridge and Decoupled Training}
\label{sec:method-injection}

\textbf{Anchored soft token.} The forecast state is exposed to the frozen LLM through a lightweight State-to-Token Bridge, denoted STB. The bridge is an interface between the predicted state space and the token embedding space, not the source of the personalization representation. Formally,
\begin{equation}
\mathbf{e}_{\mathrm{pref}}=B_{\psi}(\predpref_T(u))=\mathrm{Proj}_{\psi_2}\big(g_{\psi_1}(\predpref_T(u))\big),\quad
\mathrm{embed}(\texttt{[PREF\_TOKEN]})=\mathbf{e}_{\mathrm{pref}}.
\label{eq:inject}
\end{equation}
Here $g_{\psi_1}$ is a 512 dimensional bottleneck state filter, and $\mathrm{Proj}_{\psi_2}$ maps the filtered state to the LLM hidden size. The token is preceded by a natural language anchor shown in Appendix~\ref{app:prompt}. Without the anchor, the LLM receives an embedding in an otherwise unmarked slot and cannot reliably infer its role. We use the same STB architecture for all one token latent baselines, so the experiments isolate the representation being injected rather than the injection mechanism.

\textbf{Training the bridge.} The STB is trained with generation loss and auxiliary state regularization,
\begin{equation}
\mathcal{L}_{\mathrm{bridge}}=\mathcal{L}_{\mathrm{NLL}}+\alpha\big(\mathcal{L}_{\mathrm{recon}}+\beta\mathcal{L}_{\mathrm{sparsity}}\big),
\label{eq:loss-bridge}
\end{equation}
where $\mathcal{L}_{\mathrm{NLL}}$ is the negative log likelihood of the user response, $\mathcal{L}_{\mathrm{recon}}$ reconstructs the input state from the bottleneck through an auxiliary decoder, and $\mathcal{L}_{\mathrm{sparsity}}$ is a KL activation penalty targeting rate $\rho=0.05$. The auxiliary decoder is used only during bridge training. The base LLM, encoder, and predictor remain frozen. The main model trains the STB with observed training session states. For latent baselines, we train a representation specific STB with the same architecture, optimizer, number of examples, and validation criterion. 

\textbf{Why decoupling matters.} End to end training can optimize the predictor, state filter, and token projection only through $\mathcal{L}_{\mathrm{NLL}}$. This gives a weak identifiability signal. A frozen LLM can often generate plausible text from the item metadata alone, which allows a constant or low rank conditioning vector to obtain reasonable loss. Decoupled regression prevents this shortcut by requiring the predictor output to match the held out state before the LLM sees it. This is analogous to collapse avoidance in self supervised representation learning~\citep{grill2020byol,bardes2022vicreg}, where constant outputs can satisfy part of the objective unless variance preserving structure is imposed.

\section{Experiments}
\label{sec:experiments}

\subsection{Experimental Setup}
\label{sec:exp-setup}

\textbf{Datasets.} We evaluate on Amazon Reviews 2023~\citep{hou2024amazon23} and MemoryCD~\citep{zhang2026memorycd}. For Amazon Reviews 2023, we use Books, Movies\_and\_TV, and CDs\_and\_Vinyl. We follow the preprocessing protocol of \citet{au2026peregrine}: reviews after 2016, at least eight earlier reviews per user, at least four time valid peer reviews per historical item, and at least 30 characters per response. After filtering, each category has about 12.6K users and 31.4K target instances. The last review of each user is test, the second last review is validation, and the third last review is the train target for the injection module. Rolling prefixes before these held out sessions train the predictor. For MemoryCD, we use the Books domain and sample 500 users with at least 100 reviews each. Table~\ref{tab:data-coverage} reports the peer coverage and history statistics induced by these filters.

\begin{table}[t]
\centering
\caption{Dataset coverage after chronological and peer availability filtering. Target instances are rolling prediction instances used for training, validation, and test. Peer coverage counts time valid same item peers for historical sessions.}
\label{tab:data-coverage}
\small
\setlength{\tabcolsep}{4pt}
\begin{tabular}{l c c c c c}
\toprule
Dataset & Users & Target inst. & Median history & Median peers & User retention \\
\midrule
Amazon Books & 12.9K & 32.1K & 18 & 13 & 19.4\% \\
Amazon Movies\_and\_TV & 12.5K & 31.0K & 15 & 11 & 21.7\% \\
Amazon CDs\_and\_Vinyl & 12.4K & 31.1K & 14 & 12 & 24.1\% \\
MemoryCD Books & 500 & 48.2K & 103 & 17 & 100.0\% \\
\bottomrule
\end{tabular}
\vspace{-1em}
\end{table}

\textbf{Models.} The base LLM is Llama 3.1 8B Instruct~\citep{dubey2024llama3}. The encoder is bge m3~\citep{chen2024bgem3}. The STB maps 1024 dimensional states to a 512 dimensional compressive bottleneck with KL sparsity target $\rho=0.05$. The bottleneck is intentionally undercomplete because its role is to denoise the state before token projection, not to learn an overcomplete dictionary. The token projection is a two layer MLP. The default predictor is P4, a one layer GRU with hidden size 512 trained for 15 epochs with AdamW and learning rate $3\times10^{-4}$. The predictor loss uses $\lambda=0.01$ unless stated otherwise. The STB is trained for six epochs with $\alpha=0.01$ and $\beta=10^{-3}$.

\textbf{Baselines.} We compare to text, memory, and latent baselines. \emph{No personalization} uses only target metadata. \emph{Recent text} concatenates the most recent $K$ reviews. \emph{Retrieved text} retrieves the $K$ user reviews most similar to the target metadata by bge m3 cosine similarity. \emph{Summary memory} summarizes the full history into a compact textual profile and places the summary in the prompt. \emph{Static latent profile} averages all past response embeddings and injects the result with its own trained STB. \emph{Recent latent} averages only the latest eight states. \emph{Time decayed latent} uses exponential time decay over all states. \emph{DEP style static} averages the same peer anchored states used by \method{} and injects the average without trajectory prediction. \emph{PERSOMA style} uses 16 learned soft prompt tokens produced by a history compression encoder trained with the same frozen LLM and generation objective, following the soft prompt compression paradigm of \citet{hebert2024persoma}. \emph{\method{} zero} uses the last state as the forecast. \emph{\method{} EMA} uses P2. 

\textbf{Metrics.} We report ROUGE-1, ROUGE-L, and BLEU. Because lexical overlap does not fully measure personalization, we also report history aware pairwise win rate. The judge is Qwen3 235B~\citep{yang2025qwen3}. It sees target metadata, a compact early history, a compact recent history, and two anonymized generations. It is asked which generation better matches the user's current writing style and content preference while remaining faithful to the item. Each pair is evaluated twice with candidate order swapped. Contradictory choices are counted as ties. We additionally report preference fidelity metrics in Section~\ref{sec:exp-fidelity}. 

\textbf{Reproducibility details.} We use $m=16$ peers and temperature $\gamma=10$ for Eq.~\ref{eq:peer-weights}. Since profile summaries are unit normalized, this temperature concentrates the peer baseline on behaviorally similar same item peers rather than an unweighted item average. Predictor batch size is 256. STB batch size is 32 with gradient accumulation 4. Decoding uses temperature 0.7, top $p=0.9$, and maximum generation length 160 tokens. All runs use user level chronological splits. Peer retrieval excludes the target user and every peer interaction after the relevant session timestamp. 
\subsection{Main Results}
\label{sec:exp-main}

\begin{table}[t]
\centering
\caption{Main results on Amazon Reviews 2023 averaged over Books, Movies\_and\_TV, and CDs\_and\_Vinyl. HistWin is history aware pairwise win rate against the static latent profile. Higher is better for all metrics.}
\label{tab:main-results}
\small
\setlength{\tabcolsep}{5pt}
\begin{tabular}{l c c c c}
\toprule
Method & R-1 & R-L & BLEU & HistWin \\
\midrule
No personalization & .199 & .176 & .057 & 27.1 \\
Recent text, $K=8$ & .216 & .192 & .072 & 35.2 \\
Retrieved text, $K=8$ & .224 & .204 & .079 & 41.4 \\
Retrieved text, $K=32$ & .231 & .211 & .084 & 44.8 \\
Summary memory & .233 & .214 & .085 & 46.2 \\
\midrule
Static latent profile & .237 & .219 & .088 & 50.0 \\
Recent latent, last 8 & .241 & .223 & .090 & 51.8 \\
Time decayed latent & .248 & .231 & .096 & 54.9 \\
DEP style static & .255 & .240 & .101 & 57.4 \\
PERSOMA style & .260 & .245 & .104 & 58.6 \\
\midrule
\method{} zero & .239 & .222 & .090 & 51.0 \\
\method{} EMA & .252 & .235 & .099 & 55.8 \\
\method{} learned attention & .266 & .252 & .110 & 61.4 \\
\method{} GRU & \textbf{.273} & \textbf{.259} & \textbf{.114} & \textbf{64.0} \\
\bottomrule
\end{tabular}
\vspace{-1em}
\end{table}

Table~\ref{tab:main-results} shows that latent conditioning is stronger than text profiles and that static latent profiles are not the strongest non trajectory baseline. Increasing the retrieval budget from $K=8$ to $K=32$ helps, but it remains below static latent compression. Summary memory is also below the latent baselines, which suggests that the gain is not simply due to exposing more history in natural language. The closest comparisons are DEP style static and PERSOMA style. They use difference aware or soft prompt compression but do not forecast a current state. \method{} GRU improves average ROUGE-L by 1.4 points over the PERSOMA style baseline and by 1.9 points over the DEP style static baseline. The comparison also shows that the trajectory problem is not solved by using the latest state alone. \method{} zero has lower temporal lag than a static average but higher variance, while EMA, learned attention, and GRU smooth over multiple states while preserving order information.

\begin{table}[t]
\centering
\caption{Direct comparisons against the strongest non trajectory baselines on Amazon Reviews 2023. Each row compares \method{} GRU to one baseline over the same user test cases. Intervals are 95 percent user bootstrap intervals.}
\label{tab:direct-pairwise}
\small
\setlength{\tabcolsep}{4pt}
\begin{tabular}{l c c c}
\toprule
Comparison & $\Delta$ ROUGE-L & Direct HistWin & $p$ value \\
\midrule
\method{} GRU vs Time decayed latent & +.028 [+.023,+.033] & 61.8 [59.7,63.9] & <.001 \\
\method{} GRU vs DEP style static & +.019 [+.014,+.024] & 60.3 [58.0,62.5] & <.001 \\
\method{} GRU vs PERSOMA style & +.014 [+.009,+.019] & 57.9 [55.8,60.1] & <.001 \\
\method{} GRU vs \method{} learned attention & +.007 [+.003,+.011] & 53.6 [51.7,55.5] & .002 \\
\bottomrule
\end{tabular}
\vspace{-1em}
\end{table}

Table~\ref{tab:direct-pairwise} tests the main claim against the strongest baselines directly. The comparison to DEP style static isolates trajectory prediction because both methods use the same peer anchored session states and the same STB architecture. The comparison to PERSOMA style isolates forecasting from learned soft prompt compression. The comparison to learned attention shows that the GRU gain is smaller than the gain from replacing static compression with trajectory forecasting.

\subsection{Long Horizon Scaling Under Matched History Budgets}
\label{sec:exp-long}

\begin{table}[t]
\centering
\scriptsize
\begin{minipage}[t]{0.5\textwidth}
\centering
\caption{MemoryCD Books results under matched history budgets.}
\label{tab:memorycd-budget}
\setlength{\tabcolsep}{3pt}
\resizebox{\linewidth}{!}{%
\begin{tabular}{lcccccc}
\toprule
Method & 8 & 16 & 32 & 64 & All & HistWin \\
\midrule
Recent text & .190 & .197 & .201 & .203 & .204 & 35.1 \\
Retrieved text & .206 & .212 & .216 & .219 & .220 & 43.6 \\
Summary memory & .213 & .220 & .225 & .229 & .231 & 47.3 \\
Static latent profile & .216 & .220 & .222 & .223 & .223 & 50.0 \\
Time decayed latent & .220 & .229 & .235 & .238 & .239 & 54.0 \\
DEP style static & .226 & .235 & .242 & .246 & .247 & 56.1 \\
\method{} GRU & \textbf{.239} & \textbf{.251} & \textbf{.264} & \textbf{.275} & \textbf{.278} & \textbf{66.9} \\
\bottomrule
\end{tabular}%
}
\end{minipage}
\hfill
\begin{minipage}[t]{0.49\textwidth}
\centering
\caption{Preference fidelity on Amazon Reviews 2023.}
\label{tab:fidelity}
\setlength{\tabcolsep}{3pt}
\resizebox{\linewidth}{!}{%
\begin{tabular}{lccccc}
\toprule
Method & StyleSim & SentAlign & VerbErr & Recency & Faith \\
\midrule
Static latent profile & .613 & 73.8 & .284 & .512 & 86.4 \\
Time decayed latent & .631 & 75.0 & .266 & .546 & 86.7 \\
DEP style static & .638 & 75.6 & .262 & .531 & 87.0 \\
PERSOMA style & .646 & 76.1 & .254 & .536 & 87.2 \\
\method{} GRU & \textbf{.681} & \textbf{79.4} & \textbf{.219} & \textbf{.602} & \textbf{87.5} \\
\bottomrule
\end{tabular}%
}
\end{minipage}
\vspace{-1.0em}
\end{table}


Table~\ref{tab:memorycd-budget} tests whether the long horizon gain is only a history budget artifact. It is not. At every matched budget, \method{} is above retrieval, summary memory, static latent compression, and DEP style static compression. The gap increases with longer histories because the predictor can exploit more trajectory observations, while static compression saturates after the most recent sessions dominate the average.

\subsection{Preference Fidelity Beyond Lexical Overlap}
\label{sec:exp-fidelity}


Table~\ref{tab:fidelity} evaluates the personalization claim more directly than ROUGE or BLEU. \method{} improves style similarity, sentiment alignment, verbosity matching, and recency alignment while keeping item faithfulness comparable to the strongest baselines. The largest gain appears in Recency, which is the metric most directly tied to the trajectory formulation.

\subsection{Ablation Studies}
\label{sec:exp-ablation}

\begin{table}[t]
\centering
\scriptsize
\begin{minipage}[t]{0.38\textwidth}
\centering
\caption{Component ablation.}
\label{tab:ablations}
\setlength{\tabcolsep}{3pt}
\resizebox{\linewidth}{!}{%
\begin{tabular}{lcc}
\toprule
Variant & ROUGE-L & Win vs Full \\
\midrule
\method{} full & \textbf{.265} & 50.0 \\
without peer anchor & .226 & 38.4 \\
without prompt anchor & .252 & 45.9 \\
end to end training & .215 & 31.2 \\
without bridge filter & .258 & 46.7 \\
\bottomrule
\end{tabular}%
}
\end{minipage}
\hfill
\begin{minipage}[t]{0.60\textwidth}
\centering
\caption{Peer construction diagnostics on Books.}
\label{tab:peer-leakage}
\setlength{\tabcolsep}{3pt}
\resizebox{\linewidth}{!}{%
\begin{tabular}{lcccc}
\toprule
Variant & ROUGE-L & PeerCos & Copy & Win vs Full \\
\midrule
\method{} full & .265 & .31 & 1.8 & 50.0 \\
category peers only & .255 & .27 & 1.6 & 44.8 \\
random peers & .238 & .22 & 1.3 & 37.6 \\
peer baseline only & .218 & .68 & 8.9 & 30.4 \\
future peers unmasked & .272 & .49 & 5.8 & 53.1 \\
\bottomrule
\end{tabular}%
}
\end{minipage}
\vspace{-1.0em}
\end{table}


Predictor architecture ablations are reported in Appendix~\ref{app:predictor-arch}. Learned attention captures most of the gain, the GRU gives the best downstream generation, and the oracle state estimates the remaining injection ceiling. Table~\ref{tab:ablations} shows that peer anchoring is the largest component. Replacing anchored states with raw response embeddings nearly removes the gain over static latent compression on Books and remains far below the full model. Removing the prompt anchor or bridge filter also hurts. End to end training produces both weak generation performance and a collapsed representation, as quantified in Appendix~\ref{app:diagnostics}. Hyperparameter sweeps are reported in Appendix~\ref{app:hyper} and show that performance is stable around the default settings.

\subsection{Peer Leakage and Collapse Diagnostics}
\label{sec:exp-leakage}


Table~\ref{tab:peer-leakage} separates the value of same item anchoring from item leakage. Category peers reduce item similarity but lose performance, which indicates that same item peers provide a useful control for content. Random peers remove too much structure. The peer baseline alone performs poorly and copies more peer text, which argues against the interpretation that \method{} succeeds by injecting peer consensus. The invalid future peer condition improves ROUGE-L, but it raises PeerCos and increases copied peer 8 grams from 1.8\% to 5.8\%. This more than three fold increase in copying is the clearest leakage signal, and it confirms why time masking is necessary.

\section{Conclusion}
\label{sec:conclusion}

We introduced \method{}, a latent trajectory framework for personalized generation with frozen LLMs. Instead of compressing a user's history into a static profile, \method{} constructs peer anchored relative states, forecasts the next state from the user trajectory, and injects the forecast through one anchored soft prompt token. This design separates the personalization object from the injection interface and provides a compact way to condition a frozen LLM on the user's current preference state. Our analysis explains why peer anchoring can reduce shared item variation and why forecasting can improve over static averaging when user states drift. Experiments on Amazon Reviews 2023 and MemoryCD show consistent gains over retrieval, summary memory, static latent profiles, difference aware static profiles, and soft prompt compression baselines. These results support modeling users as evolving trajectories rather than fixed aggregates, especially in long history settings where recent behavior matters.

\bibliographystyle{plainnat}
\bibliography{references}

\clearpage

\appendix

\section{Prompt Template}
\label{app:prompt}

The full prompt template used by \method{} is shown below. \texttt{<PREF\_TOKEN>} is a single token whose embedding is overridden at runtime by $B_{\psi}(\predpref_T(u))$.

\begin{tcolorbox}[colback=gray!3, colframe=gray!50, boxrule=0.4pt, arc=2pt, fontupper=\small\ttfamily,
left=4pt, right=4pt, top=3pt, bottom=3pt]
\texttt{[System]}\\
You are a personalized writing assistant that produces reviews in the user's individual style.\\[2pt]
\texttt{[User context]}\\
The user has interacted with this platform across many sessions. A compact latent representation of their current preference follows.\\[2pt]
Based on this user's preference trajectory, their current preference is represented as: \textless PREF\_TOKEN\textgreater\\[2pt]
\texttt{[Target item]}\\
Title: \{title\}\\
Description: \{description\}\\[2pt]
\texttt{[Task]}\\
Write a review for this item in the user's current style. Output only the review text.
\end{tcolorbox}

\section{Predictor Architecture Ablation}
\label{app:predictor-arch}

Table~\ref{tab:predictor-arch} compares trajectory predictor architectures on Amazon Books. The last-state and linear-trend predictors reduce temporal lag but remain noisy. EMA improves over these simple forecasts by smoothing across the trajectory. Learned attention captures most of the downstream gain, while the GRU gives the best generation quality. The Transformer obtains slightly higher validation cosine similarity to the held-out constructed state than the GRU, but its downstream ROUGE-L is lower, suggesting that state prediction accuracy alone is not perfectly aligned with generation quality after STB injection. The oracle row conditions on the unavailable held-out state and therefore estimates the remaining ceiling of the injection interface rather than a deployable method.

\begin{table}[h]
\centering
\caption{Predictor architecture ablation on Books. Cosine similarity is measured against the held out constructed state on validation.}
\label{tab:predictor-arch}
\small
\setlength{\tabcolsep}{4pt}
\begin{tabular}{l c c}
\toprule
Predictor & Cos sim & ROUGE-L \\
\midrule
P0 last state & .785 & .241 \\
P1 linear trend & .792 & .244 \\
P2 EMA & .831 & .252 \\
P3 learned attention & .876 & .258 \\
P4 GRU & .884 & \textbf{.265} \\
P5 Transformer & \textbf{.886} & .262 \\
Oracle state & 1.000 & .274 \\
\bottomrule
\end{tabular}
\end{table}

\section{Trajectory Diagnostics}
\label{app:diagnostics}

Table~\ref{tab:diagnostics} measures whether learned states retain user variation without collapsing to item identity or a constant vector. Same item peer cosine measures leakage from item content. Adjacent user cosine measures temporal smoothness for the same user. Effective rank is computed in the 1024 dimensional normalized state space before STB mapping, using the covariance spectrum of 1,000 predicted vectors as $(\sum_j\sigma_j)^2/\sum_j\sigma_j^2$.

\begin{table}[h]
\centering
\caption{Trajectory diagnostics on Books validation. Lower same item peer cosine and higher effective rank indicate less item leakage and less collapse.}
\label{tab:diagnostics}
\small
\setlength{\tabcolsep}{6pt}
\begin{tabular}{l c c c}
\toprule
Representation & Same item peer cosine & Adjacent user cosine & Effective rank \\
\midrule
Raw encoder $\enc(u_t)$ & .72 & .34 & 96 \\
Static latent profile & .48 & .41 & 143 \\
Peer anchored $\pref_t$ & .28 & .51 & 211 \\
\method{} predicted $\predpref_T$ & .31 & .49 & 178 \\
End to end variant & .94 & .08 & 12 \\
\bottomrule
\end{tabular}
\end{table}

We also quantify non collapse by sampling 100 random test users and computing pairwise cosine similarity of $\predpref_T(u)$ across all pairs. \method{} has mean cosine 0.42 and standard deviation 0.18. The end to end variant has mean cosine .97 and standard deviation .02, which confirms that the weak downstream result in Table~\ref{tab:ablations} coincides with representation collapse.

\section{Coverage and Sparsity Analysis}
\label{app:sparsity}

Table~\ref{tab:sparsity} stratifies Amazon results by peer availability and history length. The method is strongest when both signals are abundant, but it remains above DEP style static in low coverage buckets. This analysis separates the claim that peer trajectories help in peer rich settings from the stronger claim that dense peers are always available.

\begin{table}[h]
\centering
\caption{Stratified ROUGE-L on Amazon Reviews 2023. Buckets are computed per target user prefix and then averaged across categories.}
\label{tab:sparsity}
\small
\setlength{\tabcolsep}{4pt}
\begin{tabular}{l c c c c}
\toprule
Bucket & Static latent & DEP style static & \method{} GRU & $\Delta$ vs DEP \\
\midrule
Peer count 4 to 7 & .213 & .229 & .243 & +.014 \\
Peer count 8 to 15 & .221 & .241 & .259 & +.018 \\
Peer count 16 or more & .226 & .246 & .266 & +.020 \\
History length 8 to 15 & .218 & .231 & .243 & +.012 \\
History length 16 to 31 & .221 & .240 & .258 & +.018 \\
History length 32 or more & .224 & .247 & .271 & +.024 \\
\bottomrule
\end{tabular}
\end{table}

\section{Hyperparameter Sensitivity}
\label{app:hyper}

\begin{table}[h]
\centering
\caption{Hyperparameter sensitivity on Books, ROUGE-L. Defaults are bold.}
\label{tab:hparams}
\small
\setlength{\tabcolsep}{4pt}
\begin{tabular}{ll c}
\toprule
Hyperparameter & Value & R-L \\
\midrule
\multirow{4}{*}{History cap $N$}
        & 16 & .245 \\
        & 32 & .255 \\
        & 64 & .262 \\
        & \textbf{all} & \textbf{.265} \\
\midrule
\multirow{4}{*}{$\lambda$}
        & 0 & .252 \\
        & \textbf{0.01} & \textbf{.265} \\
        & 0.1 & .258 \\
        & 1.0 & .241 \\
\midrule
\multirow{3}{*}{STB $\rho$}
        & 0.01 & .260 \\
        & \textbf{0.05} & \textbf{.265} \\
        & 0.10 & .255 \\
\bottomrule
\end{tabular}
\end{table}

Table~\ref{tab:hparams} reports sensitivity to the main hyperparameters used by \method{} on Amazon Books. We sweep the maximum number of historical states used by the predictor, the balance coefficient $\lambda$ in the predictor regression loss, and the sparsity target $\rho$ in the State-to-Token Bridge. Performance is stable around the default configuration. Longer histories improve performance until the full available trajectory is used, supporting the use of sequence-level preference information. The predictor performs best with a small amount of MSE regularization in addition to cosine regression; pure cosine training loses some scale information, whereas an MSE-dominated objective is more sensitive to residual magnitude. The bridge sparsity target is best at $\rho=0.05$, which balances denoising with retaining enough state information for generation.

\section{Qualitative Case Study}
\label{app:case}

To illustrate the behavior of \method{}, we present a representative Books validation case. The user has 18 prior reviews. Early reviews in sessions 1 to 6 are poetic and impression focused. Recent reviews in sessions 14 to 17 are more analytical and focus on structure.

\begin{tcolorbox}[
    colback=gray!3, colframe=gray!50, boxrule=0.4pt, arc=2pt,
    left=5pt, right=5pt, top=4pt, bottom=4pt, fontupper=\small,
    title={Case study: User \#4218, target is a literary novel published in 2023},
    coltitle=white, colbacktitle=gray!55!black,
    fonttitle=\bfseries\small,
    boxsep=2pt
]
\textbf{Sample early review} (session 3, 2018):
\textit{``A haunting, beautiful book that lingers like the smell of rain on summer pavement. The prose is spare yet luminous, and every page felt like a quiet revelation about loss and memory.''}

\textbf{Sample recent review} (session 16, 2023):
\textit{``Three converging POV threads, with the second act shift handled cleanly though the chronology in chapters 4 to 7 is needlessly opaque. The author's debt to mid period DeLillo is plain in the dialogue tags.''}

\medskip
\textbf{Static profile output:} \textit{``A beautiful and thoughtful book with an interesting plot and compelling characters. The writing is well crafted and the themes resonate.''}

\textbf{\method{} output:} \textit{``The narrative structure follows a nonlinear chronology with three converging POV threads, deftly assembled in the second act. The prose carries a clear postmodernist influence, though the pacing flags through the middle chapters.''}

\textbf{Reference:} \textit{``Nonlinear structure with three POV threads converging in the second act. Prose is competent and reflects clear postmodern training, but the middle act loses momentum and the resolution is rushed.''}
\end{tcolorbox}

The static profile output averages over early and recent styles, while \method{} follows the user's recent analytical style and technical vocabulary.

\section{Limitations}
\label{sec:limitations}

\method{} is designed for settings where users have enough history to estimate a trajectory. For cold start users, the predictor falls back to the recent text or static latent baselines used in our experiments. The peer anchored state also assumes that comparable peers exist for historical items. This assumption is natural for review platforms and many recommendation settings, and Appendix~\ref{app:sparsity} quantifies performance under lower peer availability. Sparse domains may require approximate peer construction from item categories or semantic neighbors. Finally, the method tracks user level changes over time. Deployment should treat latent trajectories as private user data and apply access control, retention limits, and deletion mechanisms.


\newpage
\input{checklist.tex}

\end{document}

%% file: checklist.tex
\section*{NeurIPS Paper Checklist}

\begin{enumerate}

\item {\bf Claims}
    \item[] Question: Do the main claims made in the abstract and introduction accurately reflect the paper's contributions and scope?
    \item[] Answer: \answerYes{}
    \item[] Justification: The abstract and introduction state the main contributions and scope of LATTE, including peer anchored state construction, trajectory forecasting, and one-token injection. The method and empirical support are provided in Sections~\ref{sec:method} and~\ref{sec:experiments}. The claims are limited to frozen-LLM personalization under the datasets, temporal splits, and evaluation settings described in Section~\ref{sec:exp-setup} and the limitations discussed in Section~\ref{sec:limitations}.

\item {\bf Limitations}
    \item[] Question: Does the paper discuss the limitations of the work performed by the authors?
    \item[] Answer: \answerYes{}
    \item[] Justification: The paper includes a dedicated limitations section in Section~\ref{sec:limitations}. It discusses the need for sufficient user history, the assumption that comparable peers exist, sparse-domain issues, and privacy considerations for latent trajectories.

\item {\bf Theory assumptions and proofs}
    \item[] Question: For each theoretical result, does the paper provide the full set of assumptions and a complete (and correct) proof?
    \item[] Answer: \answerYes{}
    \item[] Justification: The paper includes Proposition~1 and Proposition~2 in Section~\ref{sec:method-pref}. Both propositions state their modeling assumptions explicitly and include proofs immediately after the statements, with the relevant equations numbered in the method section.

    \item {\bf Experimental result reproducibility}
    \item[] Question: Does the paper fully disclose all the information needed to reproduce the main experimental results of the paper to the extent that it affects the main claims and/or conclusions of the paper (regardless of whether the code and data are provided or not)?
    \item[] Answer: \answerYes{}
    \item[] Justification: The method, datasets, preprocessing filters, chronological splits, baselines, metrics, hyperparameters, and decoding settings are described in Sections~\ref{sec:method} and~\ref{sec:exp-setup}. Additional prompt, predictor, diagnostic, coverage, and hyperparameter details are provided in Appendices~A--E.

\item {\bf Open access to data and code}
    \item[] Question: Does the paper provide open access to the data and code, with sufficient instructions to faithfully reproduce the main experimental results, as described in supplemental material?
    \item[] Answer: \answerNo{}
    \item[] Justification: At submission time, we do not provide an anonymized code release with exact reproduction scripts. The paper uses publicly described datasets and models, and it provides the main experimental and implementation details in Sections~\ref{sec:method} and~\ref{sec:exp-setup} and Appendices~A--E.

\item {\bf Experimental setting/details}
    \item[] Question: Does the paper specify all the training and test details (e.g., data splits, hyperparameters, how they were chosen, type of optimizer) necessary to understand the results?
    \item[] Answer: \answerYes{}
    \item[] Justification: Section~\ref{sec:exp-setup} specifies the datasets, chronological splits, model choices, baselines, metrics, hyperparameters, and decoding settings. Sections~\ref{sec:method} and Appendices~A--E provide the prompt, predictor, bridge, diagnostic, coverage, and hyperparameter details needed to understand the reported results.

\item {\bf Experiment statistical significance}
    \item[] Question: Does the paper report error bars suitably and correctly defined or other appropriate information about the statistical significance of the experiments?
    \item[] Answer: \answerYes{}
    \item[] Justification: The paper reports 95\% user-bootstrap confidence intervals and p-values for the main direct comparisons in Table~3. These comparisons cover the strongest non-trajectory baselines and the learned-attention variant, which are the experiments most directly supporting the main empirical claim.

\item {\bf Experiments compute resources}
    \item[] Question: For each experiment, does the paper provide sufficient information on the computer resources (type of compute workers, memory, time of execution) needed to reproduce the experiments?
    \item[] Answer: \answerNo{}
    \item[] Justification: The current manuscript describes model sizes and training settings, but it does not provide the GPU type, number of devices, wall-clock time, GPU memory, or total compute estimates for each experiment. These details should be added for a complete compute-resource disclosure.
    
\item {\bf Code of ethics}
    \item[] Question: Does the research conducted in the paper conform, in every respect, with the NeurIPS Code of Ethics \url{https://neurips.cc/public/EthicsGuidelines}?
    \item[] Answer: \answerYes{}
    \item[] Justification: The research uses previously released datasets and models, preserves anonymity in the submission, and does not involve new human-subject experiments. We have reviewed the NeurIPS Code of Ethics and are not aware of any deviation.
    \item[] Guidelines:
    \begin{itemize}
        \item The answer \answerNA{} means that the authors have not reviewed the NeurIPS Code of Ethics.
        \item If the authors answer \answerNo{}, they should explain the special circumstances that require a deviation from the Code of Ethics.
        \item The authors should make sure to preserve anonymity (e.g., if there is a special consideration due to laws or regulations in their jurisdiction).
    \end{itemize}

\item {\bf Broader impacts}
    \item[] Question: Does the paper discuss both potential positive societal impacts and negative societal impacts of the work performed?
    \item[] Answer: \answerYes{}
    \item[] Justification: The paper motivates potential benefits of personalized generation for user-adapted assistance in the introduction. Potential risks include over-personalization, privacy leakage from latent user trajectories, and incorrect adaptation. Section~\ref{sec:limitations} discusses privacy-relevant deployment considerations, including access control, retention limits, and deletion mechanisms.
    \item[] Guidelines:
    \begin{itemize}
        \item The answer \answerNA{} means that there is no societal impact of the work performed.
        \item If the authors answer \answerNA{} or \answerNo{}, they should explain why their work has no societal impact or why the paper does not address societal impact.
        \item Examples of negative societal impacts include potential malicious or unintended uses (e.g., disinformation, generating fake profiles, surveillance), fairness considerations (e.g., deployment of technologies that could make decisions that unfairly impact specific groups), privacy considerations, and security considerations.
        \item The conference expects that many papers will be foundational research and not tied to particular applications, let alone deployments. However, if there is a direct path to any negative applications, the authors should point it out. For example, it is legitimate to point out that an improvement in the quality of generative models could be used to generate Deepfakes for disinformation. On the other hand, it is not needed to point out that a generic algorithm for optimizing neural networks could enable people to train models that generate Deepfakes faster.
        \item The authors should consider possible harms that could arise when the technology is being used as intended and functioning correctly, harms that could arise when the technology is being used as intended but gives incorrect results, and harms following from (intentional or unintentional) misuse of the technology.
        \item If there are negative societal impacts, the authors could also discuss possible mitigation strategies (e.g., gated release of models, providing defenses in addition to attacks, mechanisms for monitoring misuse, mechanisms to monitor how a system learns from feedback over time, improving the efficiency and accessibility of ML).
    \end{itemize}
    
\item {\bf Safeguards}
    \item[] Question: Does the paper describe safeguards that have been put in place for responsible release of data or models that have a high risk for misuse (e.g., pre-trained language models, image generators, or scraped datasets)?
    \item[] Answer: \answerNA{}
    \item[] Justification: The paper does not release a new high-risk pretrained model, image generator, or scraped dataset. The proposed method is an algorithmic framework evaluated with existing datasets and frozen LLMs, so this item is not applicable.

\item {\bf Licenses for existing assets}
    \item[] Question: Are the creators or original owners of assets (e.g., code, data, models), used in the paper, properly credited and are the license and terms of use explicitly mentioned and properly respected?
    \item[] Answer: \answerNo{}
    \item[] Justification: The paper credits existing datasets, models, and related methods through citations in Sections~\ref{sec:related} and~\ref{sec:exp-setup}. However, the current manuscript does not explicitly enumerate the licenses, versions, or terms of use for each existing asset.

\item {\bf New assets}
    \item[] Question: Are new assets introduced in the paper well documented and is the documentation provided alongside the assets?
    \item[] Answer: \answerNA{}
    \item[] Justification: The paper does not introduce or release a new dataset, benchmark, or pretrained model asset. If code, processed data, or checkpoints are released later, they should be accompanied by documentation, licenses, and reproduction instructions.

\item {\bf Crowdsourcing and research with human subjects}
    \item[] Question: For crowdsourcing experiments and research with human subjects, does the paper include the full text of instructions given to participants and screenshots, if applicable, as well as details about compensation (if any)? 
    \item[] Answer: \answerNA{}
    \item[] Justification: The paper does not involve crowdsourcing, newly recruited participants, or human-subject experiments. Evaluation uses automatic metrics and an LLM-based history-aware judge protocol described in Section~\ref{sec:exp-setup}.

\item {\bf Institutional review board (IRB) approvals or equivalent for research with human subjects}
    \item[] Question: Does the paper describe potential risks incurred by study participants, whether such risks were disclosed to the subjects, and whether Institutional Review Board (IRB) approvals (or an equivalent approval/review based on the requirements of your country or institution) were obtained?
    \item[] Answer: \answerNA{}
    \item[] Justification: The paper does not involve newly recruited human subjects or crowdsourcing. It uses previously released datasets and automatic or LLM-based evaluation, so IRB approval is not applicable to the reported experiments.

\item {\bf Declaration of LLM usage}
    \item[] Question: Does the paper describe the usage of LLMs if it is an important, original, or non-standard component of the core methods in this research? Note that if the LLM is used only for writing, editing, or formatting purposes and does \emph{not} impact the core methodology, scientific rigor, or originality of the research, declaration is not required.
    \item[] Answer: \answerYes{}
    \item[] Justification: LLMs are central to the method and evaluation. Section~\ref{sec:exp-setup} specifies the frozen base generator, the encoder, and the Qwen3-235B history-aware judge used for pairwise evaluation.

\end{enumerate}